\documentclass[letterpaper, 10 pt, conference]{ieeeconf}  % Comment this line out if you need a4paper

\IEEEoverridecommandlockouts                              % This command is only needed if 
\usepackage{graphics} % for pdf, bitmapped graphics files
\usepackage[caption=false, font=footnotesize]{subfig}
\usepackage{epsfig} % for postscript graphics files
\usepackage{mathptmx} % assumes new font selection scheme installed
\usepackage{times} % assumes new font selection scheme installed
\usepackage{amsmath} % assumes amsmath package installed
\usepackage{amssymb}  % assumes amsmath package installed
\usepackage{bm}
\usepackage{amsfonts}
\usepackage{cite}
\usepackage{algorithm}
\usepackage{algorithmic}

\newtheorem{proposition}{Proposition}

\newcommand{\etal}{\textit{et al.}}

\usepackage{todonotes}
\usepackage[normalem]{ulem} %strike through

\title{\LARGE \bf
Leveraging Multiple Environments for Learning and Decision Making: a Dismantling Use Case
}

\author{Alejandro Su\'arez-Hern\'andez$^{1}$ and Thierry Gaugry$^{2}$ and Javier Segovia-Aguas$^{1}$\\%
and Antonin Bernardin$^{2}$ and Carme Torras$^{1}$ and Maud Marchal$^{2}$ and Guillem Aleny\`a$^{1}$% <-this % stops a space
\thanks{*The research leading to these results has received funding from the EU H2020 research and innovation programme under grant agreement no. 731761, IMAGINE; the HuMoUR project TIN2017-90086-R (AEI/FEDER, UE); and AEI through the María de Maeztu Seal of Excellence to IRI (MDM-2016-0656).}% <-this % stops a space
\thanks{$^{1}$Authors are with Institut de Rob\`otica i Inform\`atica Industrial, CSIC-UPC
Llorens i Artigas 4-6,
08028, Barcelona, Spain 
        {\tt \{asuarez,jsegovia,torras,galenya\}@iri.upc.edu}}%
\thanks{$^{2}$Authors are with Univ. Rennes, INSA, IRISA, Inria, France
        {\tt \{Maud.Marchal, Thierry.Gaugry, Antonin.Bernardin\}@inria.fr}}%
}

\graphicspath{{figures/}}

\begin{document}

%%%%%%%%%%%%%%%%%%%%%%%%%% Copyright page %%%%%%%%%%%%%%%%%%%%%%%%%%%%%

\onecolumn

\vspace*{\fill}

\noindent {\Large \bf Copyright notice}

\vspace{1em}

\noindent \copyright{} 2020 IEEE.  Personal use of this material is permitted.  Permission from IEEE must be obtained for all other uses, in any current or future media, including reprinting/republishing this material for advertising or promotional purposes, creating new collective works, for resale or redistribution to servers or lists, or reuse of any copyrighted component of this work in other works.

\vspace{4em}

\noindent {\Large \bf Citation of this article}

\vspace{1em}

\noindent This paper has presented during the IEEE/RSJ IROS 2020 conference. Please, in order to cite this work, refer to the conference version in IEEEXplore: {\tt https://ieeexplore.ieee.org/document/9341182}

\vspace{2em}

\vspace*{\fill}

%%%%%%%%%%%%%%%%%%%%%%%%%%%%%%%%%%%%%%%%%%%%%%%%%%%%%%%%%%%%%%%%%%%%%%%

\twocolumn

\maketitle
\thispagestyle{empty}
\pagestyle{empty}

%%%%%%%%%%%%%%%%%%%%%%%%%%%%%%%%%%%%%%%%%%%%%%%%%%%%%%%%%%%%%%%%%%%%%%%%%%%%%%%%
\begin{abstract}
Learning is usually performed by observing real robot executions. Physics-based simulators are a good alternative for  providing highly valuable information while avoiding costly and potentially destructive robot executions. We present a novel approach for learning the probabilities of symbolic robot action outcomes. This is done leveraging different environments, such as physics-based simulators, in execution time. To this end, we propose MENID~(\textit{Multiple Environment Noise Indeterministic Deictic}) rules, a novel representation able to cope with the inherent uncertainties present in robotic tasks. MENID rules explicitly represent each possible outcomes of an action, keep memory of the source of the experience, and maintain the probability of success of each outcome. We also introduce an algorithm to distribute actions among environments, based on previous experiences and expected gain.
Before using physics-based simulations, we propose a methodology for evaluating different simulation settings and determining the least time-consuming model that could be used while still producing coherent results. We demonstrate the validity of the approach in a dismantling use case, using a simulation with reduced quality as simulated system, and a simulation with full resolution where we add noise to the trajectories and some physical parameters as a representation of the real system. 

%In complex planning domains, the number of expanded states are counted from thousands to millions, so the performance of simulated actions is a key factor when planning with simulators. In this research, we aim to find the tradeoff between object resolution and computational resources for simulating scenes, such that high level action models represented with NID rules are as realistic as possible and can be exploited by a planner to produce the best plan.
\end{abstract}

%%%%%%%%%%%%%%%%%%%%%%%%%%%%%%%%%%%%%%%%%%%%%%%%%%%%%%%%%%%%%%%%%%%%%%%%%%%%%%%%
\section{Introduction}

%\GA{IRISA: Add some comments about simulation}

Probabilistic propositional planning~\cite{littman1997probabilistic} consists in computing behavioral policies reasoning over a set of actions whose outcomes are uncertain. This policy is meant to maximize a reward score. Since reward maximization subsumes goal satisfaction, probabilistic planning can be considered an extension over classical AI planning.
% , in which actions always result in deterministic effects.
Therefore, probabilistic planning can be regarded as a more powerful tool for handling the kind of task that typically arise in robotics. In this paper, we focus on the problem of learning the categorical distribution associated to the outcomes of a stochastic symbolic action, with the support of simulations.

Reinforcement Learning (RL) usually requires exhaustive interaction with the environment and results in less explainable reasoning. The work of Eppe \etal{}\cite{Eppe2019FromProblem-Solving} tackles these issues through the combination of a symbolic planner to set the subgoal roadmap, and a RL module to decide the specifics on how each subgoal is achieved. They do not tackle, however, the challenge of learning the distributions of symbolic stochastic actions.

In pure AI planning, it is often assumed that accurate models of the actions are available. This kind of knowledge can hardly be taken for granted in many real world settings. While several approaches exist for automatically acquiring deterministic models~\cite{Aineto2018a, Yang2007, Cresswell2013, Zhuo2014}, probabilistic ones require special care. Previous approaches have explored the idea of information gathering for contingent plan execution~\cite{draper1994probabilistic,pasula2007learning,martinez2017relational}. These approaches share a similar methodology: they interleave learning, reasoning and execution. However, if a robot applied this general strategy, it would be prone to put itself and its surroundings in harm's way. Lipovetzky \etal{}~\cite{Lipovetzky2015e} use simulators, like us, but their premise is based on blind search without considering the actions' model.

% Namely, at each step, the agent plans its next action reasoning over its current knowledge (possibly far off from the ground truth). Then, it executes the chosen action and gathers the changes it has produced upon the world. These data are used to refine the agent's internal model of the last executed action, and the next cycle begins.

\begin{figure}[!tb]
    % link to editable diagram: https://drive.google.com/file/d/1C0aleOFkTQxkOQryKsBpi4CNgIZ18rVt/view?usp=sharing
    \centering
    \includegraphics[width=\columnwidth]{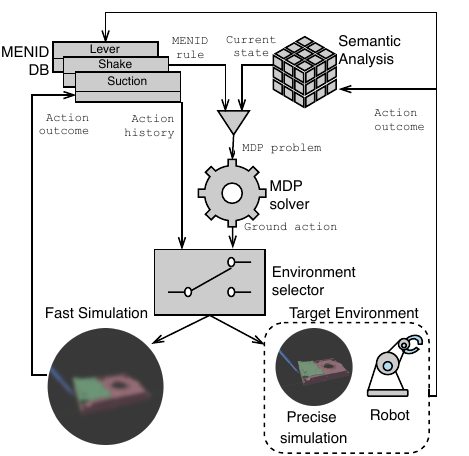}
    \caption{System's architecture. The action database contains MENID rules updated with experiences from both the simulated and the target environments. In turn, the execution history of the action is used for optimal decision making. 
    %via the resolution of a Markov Decision Process. The system is not tied to any particular solver. 
    The environment selection criteria for a given action depends on: (1) the simulator's accuracy; and (2) the confidence in the action's empirical outcome distribution.}
    \label{fig:overview}
\end{figure}

We propose a novel knowledge gathering algorithm that distributes executions among different environments. One of these environments is the \emph{target} environment (e.g. real world or highly accurate simulator). Executing actions in this environment is costly and risky. Therefore, we would like to minimize the amount of \emph{exploratory} actions performed within it. On the other hand, we have one or more \emph{test} environments that act as a proxy of the target one (e.g. fast simulator). Execution in a test environment is safe and, usually much less costly.

In this work, we rely on SOFA, a physics-based simulator~\cite{faure:hal-00681539} to carry out accurate predictions of the actions’ outcome. With SOFA, we are able to simulate several actions varying the different physical parameters. With high-resolution models, these simulations are very accurate, although also very time-consuming. A trade-off between accuracy and computation time performance must be determined.
% We propose an algorithm that optimizes the use of the simulated actions balancing accuracy and performance. At the same time, we achieve efficient learning of the robot's actions.

% \MM{ I don't know if the following sentences are fundamental=> to be mixed with the last sentence? => : but also they do not score reward. The robot's goal is to maximize reward, avoid risk, and learn the probability distribution of the actions' outcome. These objectives are not independent: reward maximization requires a well-informed policy and avoiding undesired outcomes.}

Our methodology is suitable to be of use in a wide range of robotic tasks. In this paper, we demonstrate it in a two-environments setting (see Fig.~\ref{fig:overview}). We focus on the use case of dismantling electro-mechanical devices. This kind of problems is subject to non-determinism in the action execution, as manipulation skills (such as levering) depend on factors that escape our control (e.g. friction, jerky motions) and are somewhat unpredictable. 

Our contributions are summarized as follows:
% The contributions of this paper range from the implementation of the physical simulator, to a novel approach on information gathering:
\begin{itemize}
    \item A novel learning algorithm that decides how to gather data from two sources to learn the probabilistic outcome of high-level actions.
    %related environments and the problem of learning the probabilistic dynamics of one of them.
    \item MENID (\textit{Multiple Environment Noise Indeterministic Deictic}) rules, an extension of NID rules~\cite{pasula2007learning} to cope with the multi-environment setting.
%    \item An algorithm for solving the learning problem, tested in a dismantling scenario in which SOFA, an accurate physical simulator, is used to predict actions' outcomes.
    \item A complete workflow illustrating our approach by using physics-based simulations for generating data for the learning process.
%    \item Physical simulator... (we have implemented, or merged techniques for levering, shacking, sucking,...)
%    \item Implemented a beta-distribution to learn probabilistic action outcomes from multiple sources.
%    \item Automatic generation of high-level domain knowledge in PPDDL that can be solved with an off-the-shelf probabilistic planner.
%    \item A theoretical contribution on the relation of information gathering and the concept of entropy.
%    \item A complete framework for learning and planning in the presence of non-determinism, tested in a real scenario for disassembly a HDD.
\end{itemize}

\section{Learning Methodology}
\label{sec:learning-methodology}

We want to enable our robot to learn and perform optimal decision making in execution time. To this end, we propose an algorithm that borrows notions from probability estimation and information theory.
% Regarding the latter, we find that the concept of mutual information is a very powerful tool to assess the accuracy of a proxy environment.

In this section, we present first a new action model meant to cope with several environments. Then, we outline the steps of our algorithm. Afterwards, we elaborate further on its key components.

\subsection{MENID rules}
\label{sec:menid-rules}

In the past, NID (\textit{Noise Indeterministic Deictic}) rules have been proposed as a mechanism to model probabilistic actions~\cite{pasula2007learning, martinez2017relational}. These rules enjoy the following advantages:
\begin{itemize}
    \item \textit{Deictic references}: variables that identify objects related to the parameters and are ground in the precondition.
    \item \textit{A noise effect} that covers unknown outcomes or outcomes that cannot be modeled explicitly.
    \item \textit{Derived predicates} that depend on the truth value of other predicates in the state and that allow expressing complex preconditions.
\end{itemize}

A rule has associated a list of outcomes, each with its probability of happening. However, this probability is bound to a single environment. To handle multiple environments, we define a new kind of rule called MENID (\textit{Multiple Environment NID}) as follows
\begin{equation}
a_r(\chi):\phi(\chi) \rightarrow \begin{cases}
{}^1\!p_{r,1} \dots {}^e\!p_{r,1}       & : \Omega_{r,1}(\chi) \\
\dots \\
{}^1\!p_{r,n_r} \dots {}^e\!p_{r,n_r}   & : \Omega_{r,n_r}(\chi) \\
{}^1\!p_{r,0} \dots {}^e\!p_{r,0}       & : \Omega_{r,0}(\chi)
\end{cases},
\end{equation}
where
\begin{itemize}
    \item $ a_r $ is the action associated to this rule.
    \item $ \chi $ is the set of variables for the rule, composed of two disjoint subsets: the action's parameters $ \chi_a $, and the rule's deictic references $ \chi_r $.
    \item $ \phi_r(\chi) $ is the preconditon: a set of predicates that must be present in the current state for this rule to trigger. When $ a_r $ is executed, only the rule whose preconditions are satisfied is applied.
    \item Each $ \Omega_{r,i}(\chi) $ is a different possible outcome (i.e. a set of predicates that are added or deleted from the state).
    \item $ \Omega_{r,0}(\chi) $ is the noise effect.
    \item Each $ {}^j\!p_{r,i} $ represents outcome $ \Omega_{r,i}(\chi) $'s probability in environment $ j $ (up to $ e $ environments).
\end{itemize}

An action can consist of several rules, but a rule is linked to just one action. This allows to conveniently define actions that, depending on the context, affect the world differently.

The distribution over outcomes of a rule is categorical. Therefore, outcome frequency when repeatedly triggering a rule results follows a multinomial distribution. If $ {}^j\!x_{r,i} $ is the absolute frequency of outcome $ i $ in environment $ j $, and the number of times that $ r $ has been triggered in $ j $ is $ \sum_i^{n_r} {}^j\!x_{r,i} $,  the empirical categorical distribution is

\begin{equation}
    {}^j\!p_{r,i} \approx {}^j\!q_{r,i} = \frac{{}^j\!x_{r,i}}{ \sum_{i=0}^{n_r} {}^j\!x_{r,i} }.
\end{equation}

In the next sections, we focus on the particular case of two environments ($ e = 2 $), being $ j = 1 $ the target environment, and $ j = 2 $ the test one. We describe our methodology to combine both sources of information to approximate the target distribution.

% \begin{equation}
    
% \end{equation}

%[TODO] Formalize the NEW definition of NID rules

\subsection{Algorithm Overview}

Algorithms to learn NID rules from experiences already exist, particularly for the case of guided demonstrations~\cite{martinez2017relational} and exogenous effects~\cite{Martinez_JMLR17}. We focus here on learning the different outcome probabilities $ {}^j\!p_{r,i} $. Thus, we assume that the system is initialized with all the MENID rules representing the actions available to the robot, along with their preconditions. For each rule, a list of outcomes (postconditions) $ \Omega_{r,i}(\chi) $. is also present.

The main procedure of our methodology is sketched in Algorithm~\ref{alg:main-loop}. It has been tailored to the particular case of one target environment and one test environment. However, it can be easily extended to accommodate more than two sources. 
Next, we describe the details of the algorithm.

\textbf{Initialization:} Set $ E $ represents the gathered experiences, and thus is initialized to the empty set (line 1).

\textbf{Approximating and solving the MDP:} The first step in each iteration of the main loop is to retrieve the current state from the target environment (line 3). Afterwards, an empirical transition model $ P_{target}(s,a,s') $, is constructed from the current MENID rules (line 4). This construction is detailed in Sec.~\ref{sec:combining-information-from-multiple-sources}. $ P_{target} $ is used to find the action best suited for the current state, solving an MDP (line 5). Our algorithm does not impose any particular solver. Some options are standard algorithms like Value Iteration or Thompson sampling~\cite{Russo2018}. There is also state-of-the-art solvers like PROST~\cite{keller2012prost}, Gourmand~\cite{kolobov2012blrtdp}, or FF-Hindsight~\cite{issakkimuthua2015hindsight}.

% [2020/07/19] Change paragraph to make explicit that delta > delta_thres means "not enough data"
\textbf{Testing an action:} Low confidence on the empirical outcome distribution inside the test environment is detrimental to the calculation of $ P_{target}(s,a,s') $ (Sec.~\ref{sec:combining-information-from-multiple-sources}). We determine that this is the case when $ \delta $, the estimated error of the empirical distribution (Sec.~\ref{sec:estimating-a-categorical-distribution}), is larger than a predefined threshold, e.g. 0.01 (line 6). Action testing is time-limited (lines 9, 11 and 13). Each experience is extracted and added to $ E $, labeled as $ test $ (lines 10 and 12). The algorithm marks tested actions (line 7), so they are not tested again in the next iteration. This avoids an overly cautious behavior. % Sec.~\ref{sec:estimating-a-categorical-distribution} details how the confidence on a distribution is assessed.

\textbf{Executing an action:} If an action is trusted enough, or if it has already been tested in the last iteration, it is scheduled for execution in the target environment. This action is unmarked so it is eligible again for testing (line 16). The action is executed (line 17) and the experience is added to $ E $ labeled as $ target $ (line 18).

\textbf{Updating MENID rules:} At the end of the iteration, MENID rules are updated according to the new experiences appended to $ E $ (line 20).

\begin{algorithm}[tb]
    \centering
    \caption{Learning-execution loop}
    \label{alg:main-loop}
    \begin{algorithmic}[1]
        \renewcommand{\algorithmicrequire}{\textbf{Input:}}
        \newcommand{\getCurrentState}{\textit{get\_current\_state}}
        \newcommand{\stepTest}{\textit{test\_action}}
        \newcommand{\execAction}{\textit{exec\_action}}
        \REQUIRE reward function $ R $, action set $ A = \{a_1, \dots, a_n\} $,
        time allotment for action testing $ T $
        \STATE $ E \leftarrow \emptyset $
        \LOOP
            \STATE $ s \leftarrow  $ \getCurrentState()
            \STATE $ P_{target} \leftarrow $ target transition derived from MENID rules
            \STATE $ a \leftarrow $ optimal action in $ s $ according to $ P_{target} $
            \IF {$ a $ not marked and $ \delta > \delta_{thres} $ } % [2020/07/19] changed "not enough data" by more formal condition
                \STATE mark $ a $
                \STATE $ t \leftarrow T $
                \WHILE { $ t > 0 $ }
                    \STATE $ s' \leftarrow $ \stepTest($ s, a $)
                    \STATE $ elapsed \leftarrow $ time spent testing $ a $
                    \STATE $ E \leftarrow E \cup \{ (test, s, a, s') \} $
                    \STATE $ t \leftarrow t - elapsed $
                \ENDWHILE
            \ELSE
                \STATE remove mark from $ a $
                \STATE $ s' \leftarrow $ \execAction($ a $)
                \STATE $ E \leftarrow E \cup \{ (target, s, a, s') \} $
            \ENDIF
            \STATE Update MENID rules according to $ E $
        \ENDLOOP
    \end{algorithmic}
\end{algorithm}

\subsection{Estimating a Categorical Distribution}
\label{sec:estimating-a-categorical-distribution}

Algorithm~\ref{alg:main-loop} needs to evaluate whether the distribution of an action in the test environment is estimated with enough confidence (line 6). Should that be the case, further testing can be skipped. Sec.~\ref{sec:menid-rules} already described the simplest way to estimate the distribution's parameters. However, we also need to evaluate the quality of the estimation.

Thompson~\cite{thompson1987sample} gives a method to find sample size so
\begin{equation}
Pr(|{}^j\!p_{r,i}-{}^j\!q_{r,i}| > \delta, 0 \leq i \leq n_r) \leq \epsilon,
\end{equation}
where $ {}^j\!q_{r,i} $ is an estimation of $ p_{r,i} $, $ \delta $ is the maximum error desired for the parameters of the categorical distribution, and $ 1 - \epsilon $ is the confidence on having the estimation error within $ \delta $ tolerance.

For each $ \delta, \epsilon $ and number of outcomes ($ n_r + 1$), Thompson's approach calculates the required sample size by approximating the estimation of each parameter by a Gaussian distribution. It assumes the worst-case scenario $ {}^j\!p_{r,i} = \frac{1}{1+n_r} $. This results in conservative estimates of the sample size.

We work the other way around: given the vector of frequencies $ \boldsymbol{f} = [ {}^j\!x_{r,0} \dots {}^j\!x_{r,n_r} ] $ and the desired $ \epsilon $, we calculate a $ \delta $ bound. We stop drawing observations when $ \delta $ gets below a certain threshold.

Obtaining $ \delta $ analytically is difficult. However, it can be easily approximated with the help of sampling. It is known that the conjugate prior of the categorical distribution is the \textit{Dirichlet distribution}, $ \mathrm{Dir}(\boldsymbol{\alpha}) $, where $ \boldsymbol{\alpha} = [ 1+{}^j\!x_{r,0} \dots 1+{}^j\!x_{r,n_r} ] $. This assigns a likelihood to each choice of $ {}^j\!p_{r,i} $ parameters in base to the observations. The particular case of $ n_r = 1 $ (the noise effect and a regular outcome) corresponds to the Beta distribution.

The full process for approximating $ \delta $ is described in Algorithm~\ref{alg:delta-bound}. A Dirichlet distribution $ D $ over the parameters is defined (lines 1). Then, $ D $ is sampled $ S $ times (lines 3, 4). For each sampled set of parameters, the error with respect to $ D $'s mode is computed and stored in the $ \boldsymbol{e} $ vector (line 7). Finally, the $ q $th error in the sorted $ \boldsymbol{e} $ vector is picked, where $ q $ is the index that corresponds to quartile $ (1 - \epsilon) \cdot S $ (lines 9, 10, 11). If $ S $ is sufficiently high, $ e_q $ is a very accurate approximation to $ \delta $. Fig.~\ref{fig:delta-over-time} shows two examples in which this is the case.

\begin{algorithm}[tb]
    \centering
    \caption{$ \delta $ bound calculation}
    \label{alg:delta-bound}
    \begin{algorithmic}[1]
        \renewcommand{\algorithmicrequire}{\textbf{Input:}}
        \renewcommand{\algorithmicensure}{\textbf{Output:}}
        \newcommand{\sort}{\textit{sort}}
        \newcommand{\round}{\textit{round}}
        \newcommand{\sample}{\textit{sample}}
        \newcommand{\return}{\textbf{return}}
        \REQUIRE confidence $ \epsilon $, vector $ \boldsymbol{\alpha} = [ 1+x_0 \dots 1+x_n ] $, sample size $ S $
        \ENSURE $ \delta $ s.t. $ Pr(|p_i-q_i| > \delta, 0 \leq i \leq n) \leq \epsilon $
        \STATE $ D \leftarrow \mathrm{Dir}(\boldsymbol{\alpha}) $
        \STATE Initialize vector $ \boldsymbol{e} = [e_1 \dots e_S]$
        \FOR{$ 1 \leq j \leq S $}
            \STATE $ \boldsymbol{p'} = [ p'_1 \dots p'_n ] \leftarrow $ \sample($ D $)
            \STATE $ e_j \leftarrow \max_i | p'_i - \frac{x_i}{\sum_{i'=0}^n x_{i'} } |$
        \ENDFOR
        \STATE Sort $ \boldsymbol{e} $
        \STATE $ q \leftarrow $ \round($ (1 - \epsilon) \cdot S $)
        \STATE \return{} $ e_q $
    \end{algorithmic}
\end{algorithm}

\begin{figure}[tb]
    \centering
    \subfloat[]{%
       \includegraphics[width=0.5\linewidth]{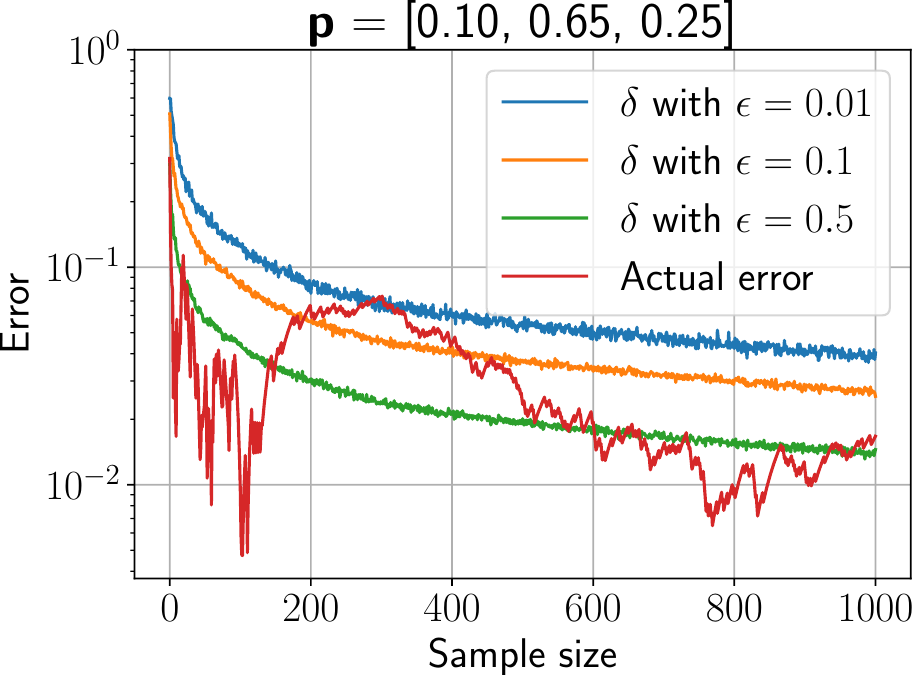}}
    \subfloat[]{%
       \includegraphics[width=0.5\linewidth]{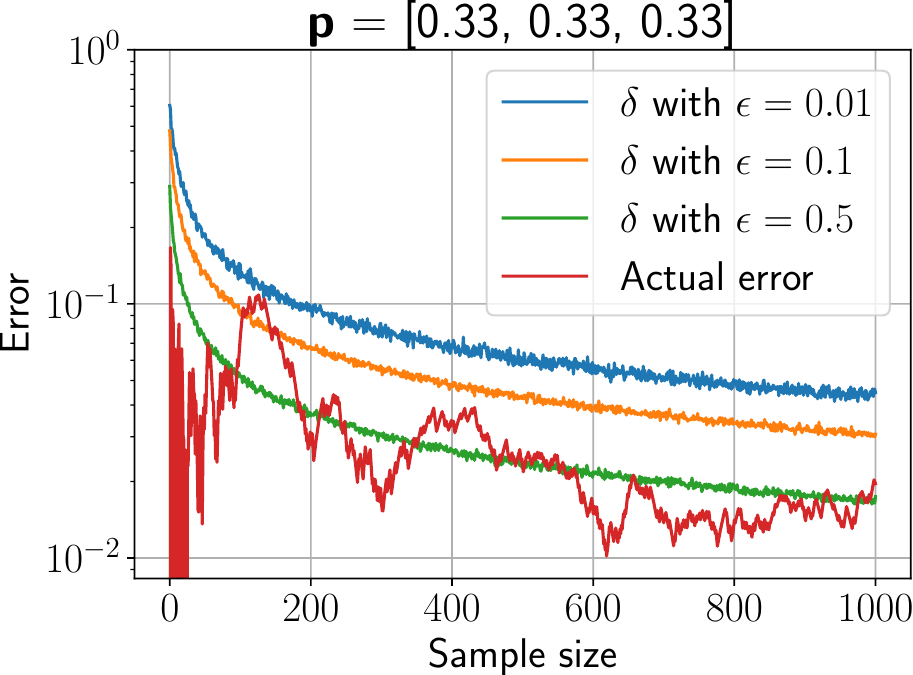}}
    \caption{Error bounds ($ \delta $) for the estimation of two categorical distributions as sample size increases. The actual error is calculated as $ \max_i | {}^j\!p_{r,i} - {}^j\!q_{r,i} | $. Notice that, in both (a) and (b), the bounds follows quite closely the trend of the error. Moreover, the bounds with $ \epsilon = 0.01 $ and $ \epsilon = 0.1 $ are seldom exceeded. The worst case scenario analyzed by Thompson (uniform distribution) is shown in (b).}
    \label{fig:delta-over-time}
\end{figure}

\subsection{Combining Information from Multiple Sources}
\label{sec:combining-information-from-multiple-sources}

Experiences come either from the test or the target environment. We need a strategy to merge these experiences. Let us first present two basic results:

\begin{proposition}
If the test environment is a perfect representative of the target one, then $ {}^1\!p_{r,i} = {}^2\!p_{r,i} $, and $ {}^1\!p_{r,i} $ can be estimated taking full advantage of the test sample frequencies:

\begin{equation}
{}^1\!p_{r, i} \approx {}^1\!q_{r, i} = \frac{{}^1\!x_{r,i} + {}^2\!x_{r,i}}{N_1 + N_2},
\end{equation}
where $ N_j = \sum_{i=0}^{n_r} {}^j\!x_{r,i} $.
\end{proposition}
\begin{proof}
$ \boldsymbol{f} = [ {}^j\!x_{r,0} \dots {}^j\!x_{r,n_r} ] $ is drawn from a multinomial distribution with parameters $ N_j, [ {}^j\!p_{r,0} \dots {}^j\!p_{r,n_r} ] $. Therefore, $ \mathbb{E}\{ {}^j\!x_{r,i} \} = {}^j\!p_{r,i} \cdot N_j $. Since $ {}^1\!p_{r,i} = {}^2\!p_{r,i} $, $ \mathbb{E}\{{}^j\!x_{r,i} \} = {}^1\!p_{r,i} \cdot N_j $. This means that $ \mathbb{E}\{ {}^1\!q_{r, i} \} = \frac{\mathbb{E}\{ {}^1\!x_{r,i} \} + \mathbb{E}\{ {}^2\!x_{r,i} \}}{N_1 + N_2} = \frac{{}^1\!p_{r,i} \cdot (N_1 + N_2)}{N_1+N_2} = {}^1\!p_{r,i} $.
\end{proof}

In general, however, the distributions of the environments are different, and the following proposition applies.

\begin{proposition}
If there is any difference between the two distributions, the best estimator for $ {}^1\!p_{r, i} $ considers only the target, since including the frequencies from an alien population would bias the estimator.
\end{proposition}
\begin{proof}
The same reasoning as before can be used to see that, when $ {}^1\!p_{r, i} \neq {}^2\!p_{r, i} $, $ \mathbb{E}\{ {}^1\!q_{r, i} \} = \frac{N_1}{N_1+N_2} \cdot {}^1\!p_{r,i} +  \frac{N_2}{N_1+N_2} \cdot {}^2\!p_{r,i} $. Therefore, by the law of large numbers, the best estimator of $ {}^1\!p_{r, i} $ when $ N_1 \to \infty $ uses only samples from $ j = 1$.
\end{proof}

However, when $ N_1 $ is small (as it is initially), and assuming that the test environment distribution is reasonably close to the target one, the outcome frequencies in the test environment may help establish a prior for the target probabilities. If the test frequencies are given a weight that decreases with $ N_1 $, we can ensure that, in the limit, our estimation will be unbiased. This same principle motivates the \textit{decreasing-m-estimate}\cite{martinez2015planning} that we reformulate for our purposes:

% [2020/07/19] Question about the square root (Reviewer #4). I don't get what it is that they don't understand.
\begin{equation}
{}^1q_{r, i} = \frac{{}^1\!x_{r,i} + \frac{m}{\sqrt{1+N_1}} {}^2\!x_{r,i}}{N_1 + \frac{m}{\sqrt{1+N_1}} N_2},
\end{equation}
where $ m $ is the weight given initially to the test executions.

When $ N_1 \to 0 $, the estimation will be based almost entirely on the test executions. Conversely, when $ N_1 $ increases, the weight for the test executions decreases, basing the estimation more and more on the target executions.

Line 20 of Algorithm~\ref{alg:main-loop} updates the $ {}^1\!p_{r,i} $ probabilities of the MENID rules according to the decreasing-m-estimate, while the test environment rules are updated with the simple estimate described in Sec.~\ref{sec:menid-rules}. The most updated probabilities are always available for planning when constructing the transition model $ P_{target}(s,a,s') $ in line 4. Specifically, $ P_{target}(s,a,s') =  {}^1\!q_{r,i} $ if rule $ r \in a $ triggers in $ s $, and $ s \xrightarrow[\Omega_{r,i}]{} s'$ (i.e. $ s' $ follows after the $ i $th outcome). If the transition $ s \xrightarrow[a]{} s' $ is impossible, $ P_{target}(s,a,s') = 0 $.

% \begin{equation}
% P_{target}(s,a,s') = \begin{cases}
%     0 & \mbox{if transition $ s \xrightarrow[a]{} s' $ is impossible} \\
%     {}^1\!q_{r,i} & \mbox{if $ r \in a $ triggers in $ s $, and $ s \xrightarrow[\Omega_{r,i}]{} s'$}
% \end{cases}
% \end{equation}

%[TODO] Reuse Alejandro's work from GA5 (definitions)
%[TODO] Cite ASC papers for affordances in simulations, and maybe reuse experiments from GA5?

\section{Simulating Robotic Actions}
\label{sec:simulating}

% \subsection{Physics-based simulations}

In order to handle accurate simulations at a near-interactive computation time, we chose SOFA Framework~\cite{faure:hal-00681539}.
SOFA offers the possibility to handle accurate simulations of rigid and deformable objects at interactive time performance. Therefore, it represents a good alternative to more traditional mechanical simulators that do not handle the object dynamics at a reasonable framerate, as well as game engines that do not handle accurate physics. 
As it is an open-source software, we were also able to design the scenarios as well as controlling the parameters as much as we wanted.

\subsection{Simulated Actions}

Within the use case of dismantling electro-mechanical devices, we chose three different robotic actions meant for disassembling a hard drive. The three actions are: (1) levering the Printed Circuit Board (PCB) or from the bay; (2) shaking the bay to remove the PCB; and (3) sucking the PCB with a suction cup and removing it from the bay. We enrich actions with parameters that give nuance over their physical realization and that could influence their success probability (e.g. discrete values of direction and force for \textit{lever}).

These three actions are subject to non-determinism in the action execution, as manipulation skills depend on factors that escape our control in the real environment
(e.g. friction, jerky motions, errors in perception). Simulating these actions require the use of a simulator able to handle the simulations of complex scenes generated from real data.

The resolution of the meshes is a key component for obtaining accurate results, at the price of a high computation time. Therefore, we present in Sec.~\ref{sec:meshResolution} a study evaluating the impact of the model quality on the symbolic error. 

\subsection{The Resolution Problem}

One of the first concerns that arise when using simulation as a proxy environment is realism and performance. On the one hand, we would like the simulator to behave as close as possible as the real world. For that, we need to include detailed 3D meshes that model the real life component. On the other hand, we want the simulator to be efficient enough so it can execute a large number of simulations, and detailed mesh models make collision detection more complex.

At our disposal we have a library of 3D scans of hard drive components that we use to test our approach. We have full control on the detail of such models. We deem sensible to measure the impact of the downgrade. We perform a test on the accuracy and efficiency of the simulated actions with respect to the most detailed model (Sec.~\ref{sec:meshResolution}).

\section{Experimental Study}
\label{sec:probabilistic-planning}

First, we show that simulations with models of different quality yield substantially different qualitative results. On the one hand, this imposes the need to compromise performance to accuracy and vice versa. On the other hand, this shows that simulations with different qualities could already be considered different environments.

Second, we provide evidence that Algorithm~\ref{alg:main-loop} takes advantage of fast simulations to quickly learn the outcomes probabilities and to exploit the best actions in the target environment. We show this by plugging in a high-quality SOFA simulation as target environment.

%\subsection{The simulator setting} 
%[TODO] Add new info special for this case or remove it and refer to Simulator section?
%[TODO] Explain all new actions that performs deterministically (unscrewing, flip, ...) and are relevant for task planning

%\subsection{Computing high-level actions}
%[TODO] Given a resolution from previous sections, learn the NID rules and compile everything into an action scheme with costs

\subsection{The Mesh Resolution Problem}
\label{sec:meshResolution}

%[TODO] Modelling learning action costs as multiarm bandint problem, explaining the exploration/exploitation algorithm and how it performs for each action [EXPERIMENTS]

\begin{figure}[tb]
    \centering
    \subfloat[\label{fig:mesh-resolution-scenarios-a}]{%
       \includegraphics[width=\linewidth]{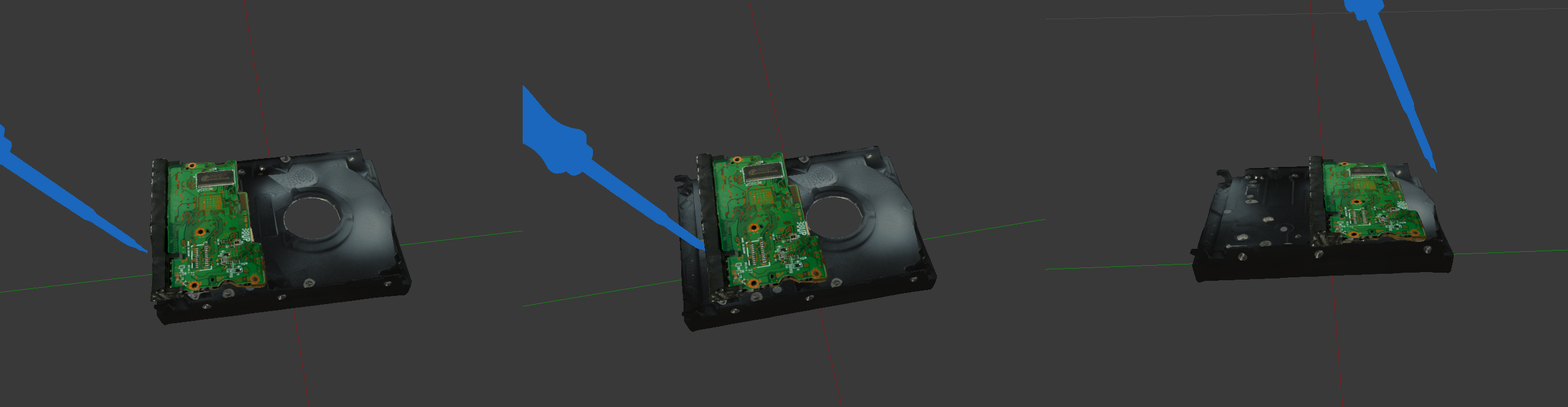}}
       
    \subfloat[\label{fig:mesh-resolution-scenarios-b}]{%
       \includegraphics[width=\linewidth]{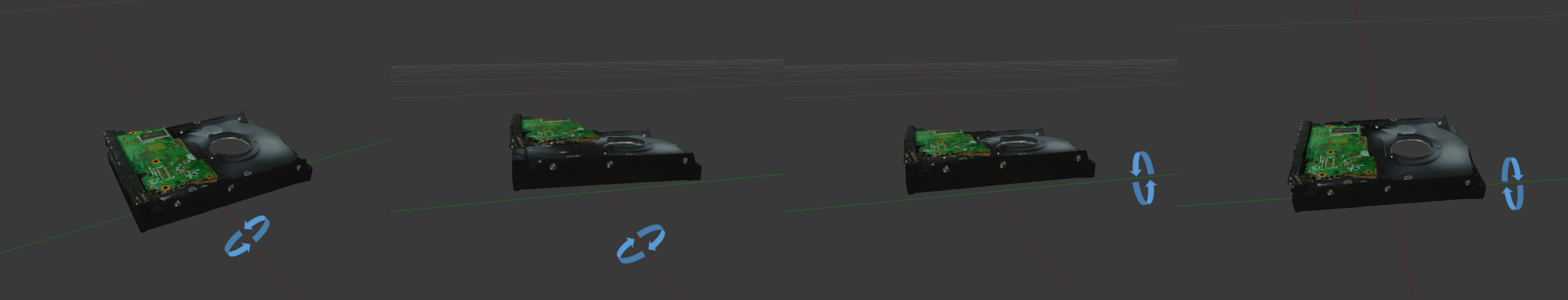}}
       
    \subfloat[\label{fig:mesh-resolution-scenarios-c}]{%
       \includegraphics[width=\linewidth]{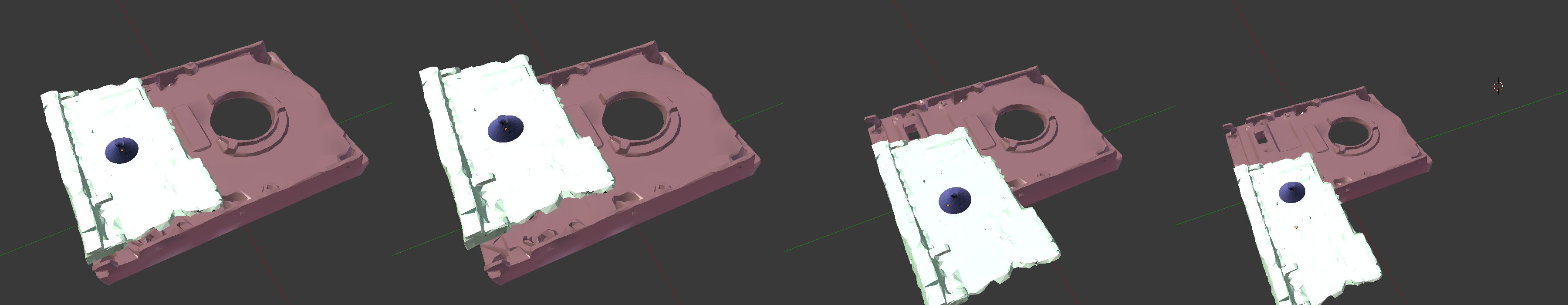}}
       
    \caption{Illustration of the three scenarios: (a) First scenario: the lever (blue) is used to remove the PCB (green) from the bay; (b) Second scenario: different shaking angles are applied to the bay to remove the PCB from it; (3) Third scenario: the PCB (white) is removed from the bay (red) using a suction cup (in black, located on the PCB).}
    \label{fig:mesh-resolution-scenarios}
\end{figure}

We have evaluated the performance/accuracy trade-off of the actions in three scenarios\footnote{The computer specifications for these experiments are: Linux Kernel 5.2.11-100 x86\_64 Fedora 29, Intel(R) Xeon(R) CPU E5-1603 v4 @ 2.80GHz, 16GiB Ram, GeForce GTX 1080.}
% PC configuration for the simulation:
% Linux Kernel 5.2.11-100 x86\_64 Fedora 29
% Intel(R) Xeon(R) CPU E5-1603 v4 @ 2.80GHz
% 16GiB Ram
% GeForce GTX 1080

\begin{figure}[tb]
    % [2020/07/15] Fixed legend so it appears at a more intuitive location
    \centering
    \includegraphics[width=1\columnwidth]{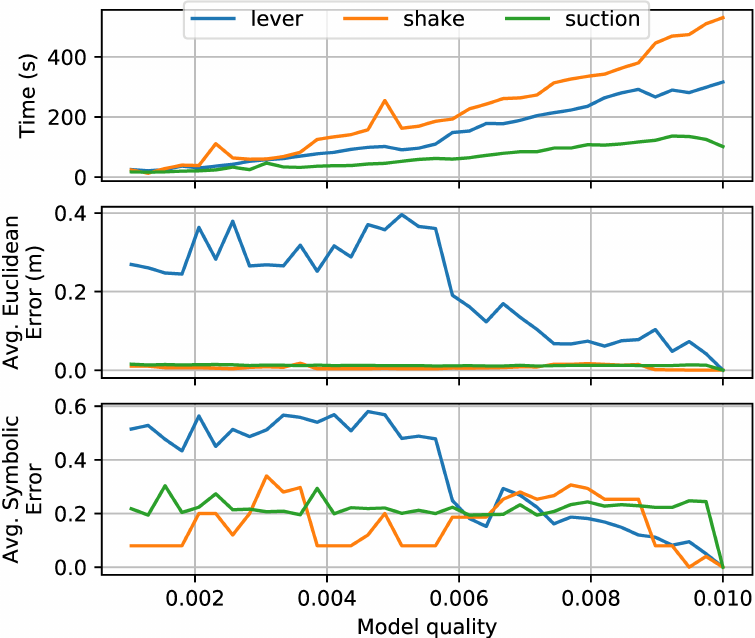}
    \caption{Different performance measures for the \textit{lever}, \textit{shake} and \textit{suck} action as the quality of the models (i.e. proportion of retained vertices) varies.}
    \label{fig:performance-vs-quality}
\end{figure}

\subsubsection{First scenario: levering a PCB} 
The first scenario considers a rigid lever that is used to remove a PCB from the bay of the hard disk (see Fig.~\ref{fig:mesh-resolution-scenarios-a}). All the objects are rigid. We used an Euler implicit integration scheme. Objects' masses were set to their real world counterpart's. In this task, a good location on the PCB to perform the levering action must be found.
% , given a friction parameter between the different objects.
In our simulations, we evaluated 20 different positions of the lever, located at the boundary of the PBC.
% , with a direction perpendicular to the boundary at that position. 

\subsubsection{Second scenario: shaking the bay}
The second scenario consists in shaking the bay to remove the PCB from it (see Fig.~\ref{fig:mesh-resolution-scenarios-b}). We used the same simulator configuration as for the first scenario. In this task, we tested for 10 different angles of shaking movement, ranging from 0.01 to 5 degrees.

%\begin{figure}[tp]
 %   \centering
%    \includegraphics[width=1\columnwidth]{figures/e2_t_annotated.png}    \caption{Second scenario: different shaking angles are applied to the bay to remove the PCB from it.}
 %   \label{fig:secondscenario}
%\end{figure}

\subsubsection{Third scenario: sucking the PCB}
In the third scenario, we used a deformable suction cup to suck the PCB and remove it from the bay (see Fig.~\ref{fig:mesh-resolution-scenarios-c}). We used the method proposed by Bernardin \etal{}~\cite{Bernardin2019} in this scenario. Both the PCB and the bay were rigid objects with the same properties as for the two first scenarios, while the suction cup was modeled as a deformable object using the co-rotational Finite Element Method (FEM). The success of the task is highly dependent on the location of the suction cups as well as the pressure applied on the PCB. In our simulations, we tested for 9 different positions spread on a $3\times3$ grid on the PCB. We also applied 3 different suctions pressures, ranging from -0.001 to -1135 Pa (difference with atmosphere pressure, 101135 Pa).

%\begin{figure}[tp]
%    \centering
%    \includegraphics[width=1\columnwidth]{figures/e3_c.png}
%    \caption{Third scenario: the PCB (white) is removed from the bay (red) using a suction cup (in black, located on the PCB).}
%    \label{fig:thirdscenario}
%\end{figure}

%\subsubsection{Mesh quality}
The mesh quality was defined as the ratio of the number of points between the reduced models and the original one.
All experiments were performed on 36 different mesh qualities, ranging from 0.00103 to 0.01 with 5 repetitions for each set of parameters. For the lowest quality, the PCB had 303 vertices and 762 triangles while the bay had 491 vertices and 1082 triangles. For the highest quality, the PCB had 3310 vertices and 6776 triangles while the bay had 5183 vertices and 10466 triangles.

Results are shown in Fig.~\ref{fig:performance-vs-quality}. The symbolic error is measured as the $ 1 - J(s,s') $, where $ s, s'$ are the collection of predicates that represent the state at  the end of a low quality simulation and its high quality counterpart, respectively; and $ J(s, s') $ is the Jaccard index between both sets:
\begin{equation}
    J(s, s') = \frac{|s \cap s'|}{|s \cup s'|}
\end{equation}

Pikes on the graphs can be explained as noise on the initial conditions of the experiment. The irregularities of the models may result in the objects being displaced.
% While this is not significant most of the time, they may end up in "better" suited positions for the task.
% It can be noted that this is significant due to the experiment not succeeding enough, in particular for experiment 2.
% A way to solve this problem would be to relaunch the experiment with better tuned parameters and longer experiments, so they are less prone to this kind of noise/failure.

The plot already allows us to identify the trend that we anticipated in the lever action, with an inflection point around 0.006: higher qualities result in more execution time and in results significantly different from lower qualities.
% For lever and suction, the drop in the error is not observed until later.

\subsection{Interleaved Learning and Planning}

\begin{figure}[tb]
    \centering
    \subfloat[\label{fig:results-a}]{%
       \includegraphics[width=\linewidth]{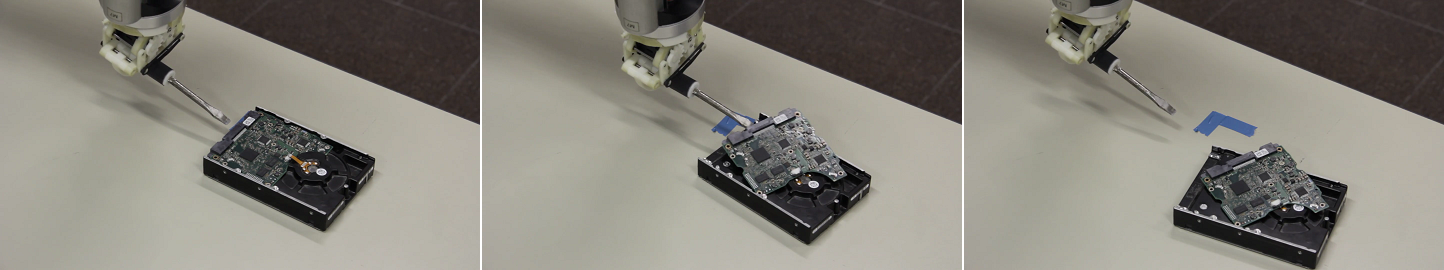}}
       
    \subfloat[\label{fig:results-b}]{%
       \includegraphics[width=\linewidth]{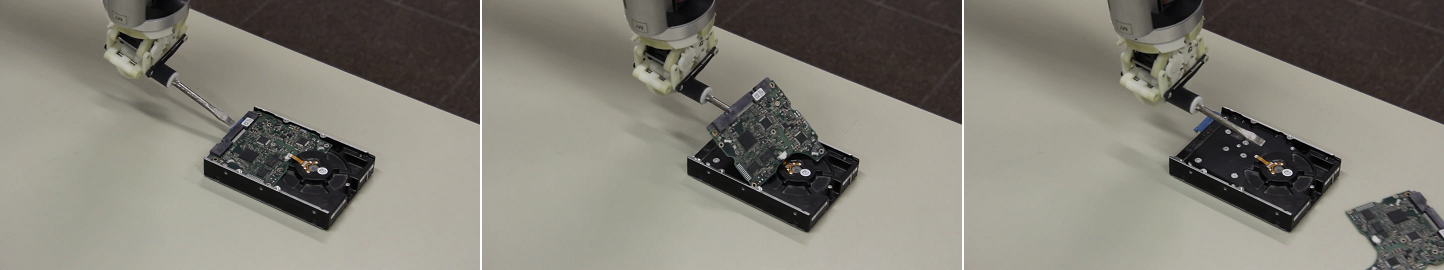}}
    \caption{Two different lever motions for levering and removing a PCB from the hard drive reproduced with the robot arm. The first motion, (a), is unsuccessful, while the second, (b), is successful.}
    \label{fig:motions-real}
\end{figure}

\begin{figure*}[tb]
    \centering
    \subfloat[\label{fig:results-sim-a}]{%
       \includegraphics[width=0.33\linewidth]{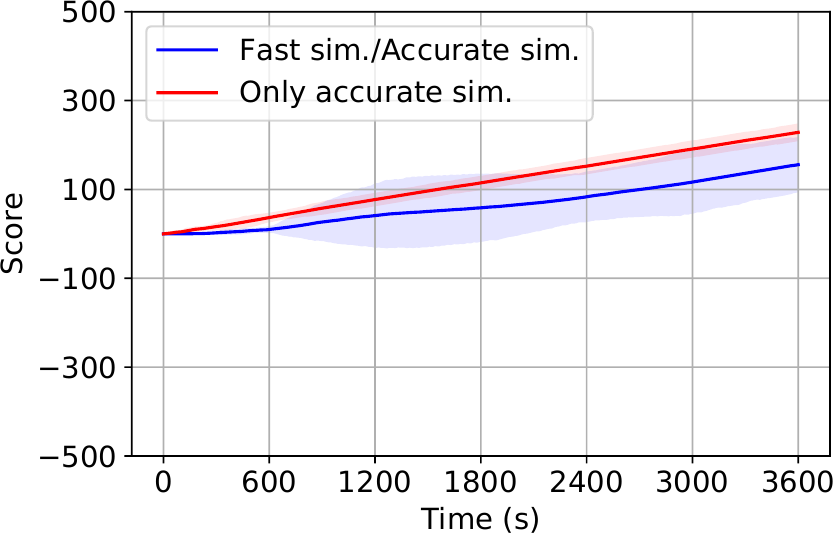}}
    \subfloat[\label{fig:results-sim-b}]{%
       \includegraphics[width=0.33\linewidth]{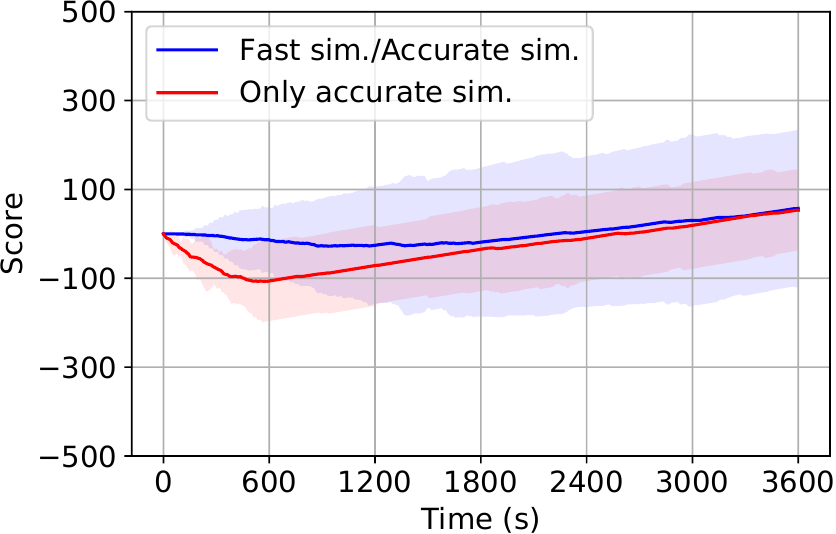}}
    \subfloat[\label{fig:results-sim-c}]{%
       \includegraphics[width=0.33\linewidth]{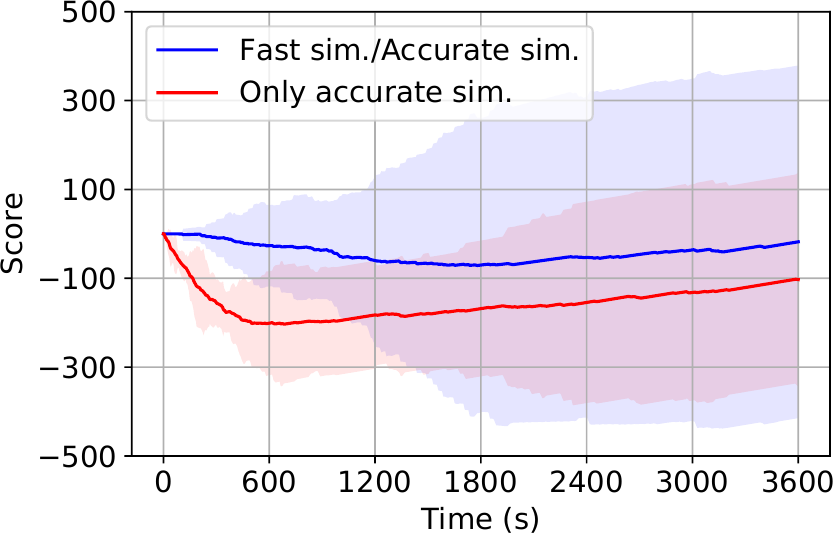}}
    \caption{Results when learning the accurate, but slow simulation environment (target environment). The reward for successful actions is always 1. In (a), there is no penalty for failing an action. In (b) and (c), failures are penalized with a decrement of 5 and 10 to the score, respectively. The blue line shows the accumulated reward when the fast but inaccurate environment is used as test environment. Red line represents a learner that can access exclusively the slow simulator. The shaded regions represent 3 standard deviations around the central trend.}
    \label{fig:results-sim}
    \vspace{-1em}
\end{figure*}

\begin{figure*}[tb]
    \centering
    \subfloat[\label{fig:results-target-a}]{%
       \includegraphics[width=0.33\linewidth]{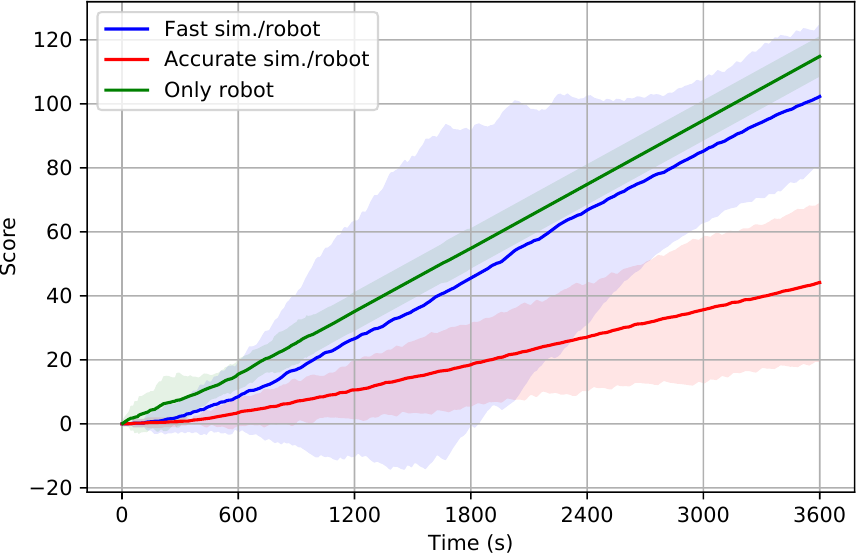}}
    \subfloat[\label{fig:results-target-b}]{%
       \includegraphics[width=0.33\linewidth]{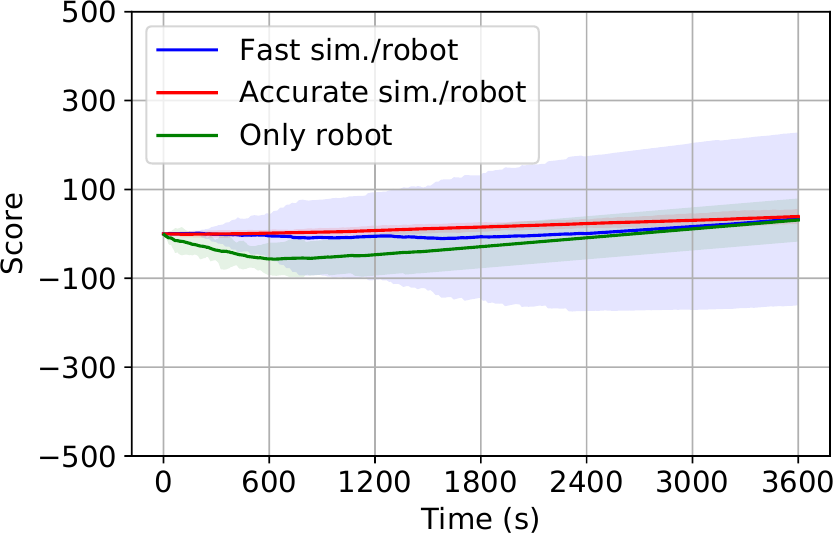}}
    \subfloat[\label{fig:results-target-c}]{%
       \includegraphics[width=0.33\linewidth]{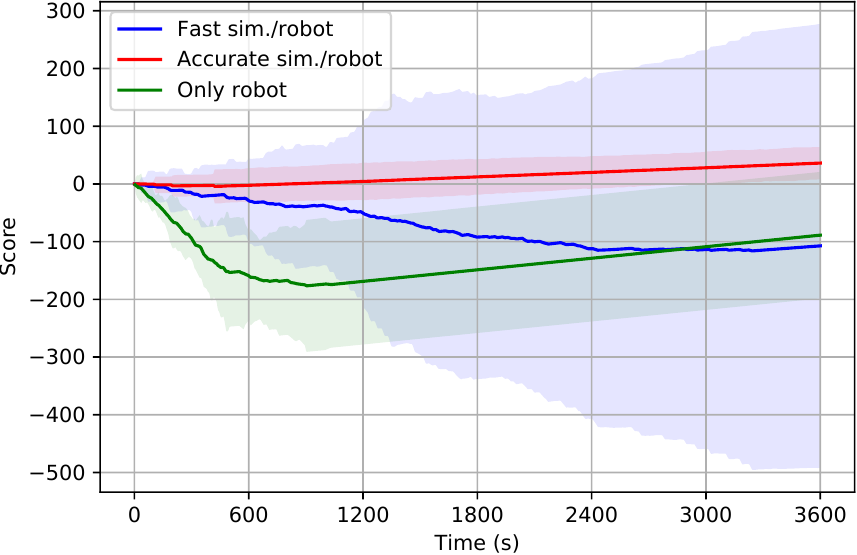}}
    \caption{Results when the target environment is the real world (i.e. physical robot acting in a real-life setting). Scenario (a) features a risk-free environment (no penalty for failures), while scenarios (b) and (c) correspond to scenarios with a penalty of 5 and 10 for failed actions, respectively.}
    \label{fig:results-target}
    \vspace{-1em}
\end{figure*}

Our learning algorithm has been evaluated using two configurations of SOFA and a real-life robot setting as the available environments. The scenes used in our evaluation contain a variety of hard-drive scenarios. The goal in each scenario is the removal of the PCB by means of the available actions (motions with different application points).

One simulation environment has low quality ($ q = 0.001 $) models. While simulations in this high-performance setting run promptly, the accuracy of the results varies sensibly with respect to the high-quality models, as depicted in Fig.~\ref{fig:performance-vs-quality}. The other simulation has high-quality models ($ q = 0.01 $). This setting takes, on average, much more time than its low-resolution counterpart. It can be appreciated that the outcomes vary sensibly, which suggests that high-quality models capture more complex interactions. To account for stochasticity, random noise is added to the input trajectories and the physical parameters (friction and restitution).

The robot environment features a WAM robotic arm in front of a table. On the table is a hard drive with the PCB side facing upwards. Several levering motions can be executed to pull the PCB out of the drive, and some of these actions are more successful than others, as shown in Fig.~\ref{fig:motions-real}.

Our baseline learns purely from the target environment. This can be achieved simply by setting $ T = 0 $ in our algorithm (this effectively disables the execution of actions in the test environment). We compare this to $ T = 20 $ (i.e. 20 seconds for testing an action). We use the \textit{m-estimate} described in Sec.~\ref{sec:estimating-a-categorical-distribution} with $ m = 10 $. The success or failure of an action is rewarded and penalized, respectively (benefit vs risk tradeoff). To evaluate, we use the accumulated reward (score) over several episodes for a total duration of 1 hour.

We have run several sets of experiments assigning different rewards to successful and unsuccessful executions of actions. In Fig.~\ref{fig:results-sim}, we show in blue the result of trying to learn the best action(s) to execute in the high-quality simulation by interleaving executions of actions in the low-quality and the high-quality simulations. The red line shows what happens when only the target environment (the high-quality simulation) is available.

For each reward choice, we have run five experiments. The solver used to calculate the best action is a custom implementation of a Thompson sampler~\cite{Russo2018}.

When there is no risk involved (i.e. the penalty for failing is 0), as in Fig.~\ref{fig:results-sim-a}, it is usually best to encourage experimentation in the target environment, since failing an action does not incur in any negative consequence and we sample results directly from the target distribution, while simulation only wastes time. However, when more relevance is given to failures, using the target environment for testing becomes too risky. Large negative reward can be accumulated, and the benefits of having a virtual environment to simulate actions and gain more information about them is more evident.

Our method works thanks to the following: (1) the test environment allows sampling the action outcomes from a probability distribution that is sufficiently similar to the target one; (2) this can be done in less time than using the target environment directly; and (3) risk-free testing avoids the occasional penalties and adverse effects that occur at the beginning, in the abscense of experience.

Fig.~\ref{fig:results-target} shows the result for settings that involve the real-life robot. We show: (1) in blue, the low-quality simulation/real robot setting; (2) in red, the high-quality simulation/real robot setting; and (3) in green, what happens when only interactions with the real robot are available (baseline). These results show that, when simulation is involved, we avoid the accumulation of failures at the beginning of the risky environments (Fig.~\ref{fig:results-target-b}, Fig.~\ref{fig:results-target-c}). However, simulation requires time, so if no penalty, or mild penalties are associated to failures, direct executions with the robot or with less precise simulations give rewards much faster (Fig.~\ref{fig:results-target-a}, Fig.~\ref{fig:results-target-b}). Since our method will eventually decide to not simulate anymore, the end slope of the curves is always the same.

%[TODO] Once we have the set of actions, compute a plan that minimizes a cost or maximizes a reward. If it can be executed deteministically, it will perform better but is realistic?
%[TODO] Should we compute only the plan? Can we compute a policy? (FOND), Can we maximize a reward with more expressive languages? (RDDL), ...
% \subsubsection{Probabilistic values as action costs} % PDDL (FD) 
% \subsubsection{Relaxing probabilities to non-determinism} % FOND (PRP)
% \subsubsection{Action centric probabilistic planning} % PPDDL (MIN-GPT)
% \subsubsection{Transition centric probabilistic planning} % RDDL (PROST)

%\subsection{Testing in a realistic environment}

\section{Conclusions and Future Work}
\label{sec:conclusions}
Motivated by the idea of using a simulator to gather symbolic information about the success ratio of robot actions, we have presented a framework that allows to learn from experiences performed in different environments. We have introduced MENID rules, a formalism to model different outcomes according to multiple environments and their uncertainty. We have also presented an algorithm able to decide in which environment execute the actions.

%We have proposed an algorithm for a robot that operates in a dual-environment framework. Our goal is to try to act rationally, while learning the actions' probability distribution over outcomes. One of the environments acts as a test sandbox, a fast simulator where actions can be executed with no risk nor benefit. The other environment is the target one, where actions may have a reward or penalization associated. Our algorithm chooses to execute actions preferentially in the test environment, until enough confidence in them is built to allow a safe execution in the target one. This was complemented with the MENID rules formalism for modeling actions in multi-environment settings.

Our framework has been evaluated with scans of real objects and actions used in a hard drive disassembly task: lever, shake, and suck. The SOFA simulator has acted as test environment in a low-resolution configuration, and as target environment in a high-resolution one by adding noise. Experiments with a real robot have been also conducted. Compared to a standard strategy that learns only from the target environment, our algorithm learns the actions' outcomes with less risk.

Once we have validated the learning of action outcomes, in the future we would like to extend the proposed algorithm to learn also new outcomes to the rules and even new rules from scratch. Finally, we would like to explore the case of more than two environments in detail.

%\section*{APPENDIX}


%\section*{ACKNOWLEDGMENT}

%The preferred spelling of the word ÒacknowledgmentÓ in America is without an ÒeÓ after the ÒgÓ. Avoid the stilted expression, ÒOne of us (R. B. G.) thanks . . .Ó  Instead, try ÒR. B. G. thanksÓ. Put sponsor acknowledgments in the unnumbered footnote on the first page.

\bibliographystyle{IEEEtran} 
\bibliography{IEEEabrv,references}
% \bibliography{IEEEabrv,mybibfile}

\end{document}